\documentclass[letterpaper, 10 pt, conference]{ieeeconf}  

\IEEEoverridecommandlockouts                              

\usepackage[ruled,linesnumbered]{algorithm2e}
\usepackage{amsmath}
\usepackage{amssymb}
\usepackage{dsfont}
\usepackage{graphics}
\usepackage{graphicx}
\usepackage{hyperref}
\usepackage{subcaption}
\usepackage{times}

\usepackage{booktabs, multirow} 
\usepackage{soul}
\usepackage[table]{xcolor} 
\usepackage{changepage,threeparttable} 

\newtheorem{theorem}{Theorem}

\newcommand\lqf{\mathcal{Q}}
\newcommand\lqffn[1]{$\lqf(#1)$}

\SetKwInOut{Notions}{Notions}

\title{\LARGE \bf
Model-free Neural Lyapunov Control for Safe Robot Navigation
}

\author{
    Zikang Xiong, 
    Joe  Eappen, 
    Ahmed H. Qureshi, 
    and Suresh Jagannathan
    \thanks{Authors affiliate to Computer Science Department, Purdue University, IN 47906, USA. (E-mail: xiong84@purdue.edu; jeappen@purdue.edu; ahqureshi@purdue.edu; suresh@cs.purdue.edu)}}

\begin{document}

\maketitle
\thispagestyle{empty}
\pagestyle{empty}

\begin{abstract}
    Model-free Deep Reinforcement Learning (DRL) controllers have demonstrated promising results on various challenging non-linear control tasks. While a model-free DRL algorithm can solve unknown dynamics and high-dimensional problems, it lacks safety assurance. Although safety constraints can be encoded as part of a reward function, there still exists a large gap between an RL controller trained with this modified reward and a safe controller. In contrast, instead of implicitly encoding safety constraints with rewards, we explicitly co-learn a Twin Neural Lyapunov Function (TNLF) with the control policy in the DRL training loop and use the learned TNLF to build a runtime monitor. Combined with the path generated from a planner, the monitor chooses appropriate waypoints that guide the learned controller to provide collision-free control trajectories. Our approach inherits the scalability advantages from DRL while enhancing safety guarantees. Our experimental evaluation demonstrates the effectiveness of our approach compared to DRL with augmented rewards and constrained DRL methods over a range of high-dimensional safety-sensitive navigation tasks.
\end{abstract}

\section{Introduction}

Conditioning the goal of a controller with waypoints generated by a planner is a natural approach to combine planning and control~\cite{lavalle2006planning}. A low-level controller guides a robot to follow a path generated by a high-level planner and finally navigates a robot to achieve the desired goal. However, such a decomposition must also take into account the possibility that the controller may not exactly follow the planned path. For example, when an autonomous vehicle needs to make a U-turn, the planned path might be too sharp to follow since the planner is usually agnostic of the underlying controller's capabilities. This raises safety concerns - even if the planner can generate safe plans, it is not always the case that the controller can follow the given path. 

Solving kinodynamic constraints~\cite{donald1993kinodynamic} addresses the inconsistency problem between planning and control of a robot. Recently, the application of DRL in this domain has provided a scalable solution for addressing kinodynamic constraints. \cite{faust2018prm,chiang2019rl,ota2020efficient} train DRL local controllers to reach a waypoint while also avoiding collisions. Since these local controllers can avoid collisions, the safety of these controllers following the planned paths is further improved. These controllers achieve obstacle avoidance by formulating collisions as penalty items in the reward function. Notably, since the avoidance capability of these controllers is only implicitly encoded as part of the reward function, the actual avoidance capability of these controllers remains a question. Moreover, it is generally hard to decide the proportion of these penalty items. When this proportion is too large, the controller is likely to learn to stay in the initial position since this is always safe and can receive a better reward compared with exploring other regions of the space, potentially leading to collisions. If it is too small, the controller will likely ignore these collisions and only focus on achieving the goal. This phenomenon, known as reward hacking, often makes training DRL policy extremely challenging in high-dimensional spaces.

The key idea of our approach is to explicitly place a sequence of robot reachability boundaries around a given plan. As long as we can ensure that no obstacle appears in the reachability boundary, we can provide enhanced safety guarantees when following a given plan. Lyapunov function and its region of attraction are widely used to build reachability boundaries. Combined with Neural Lyapunov Functions, (NLF)~\cite{richards2018lyapunov,chang2020neural,sun2020learning} which has been extensively studied, this approach has the potential to scale such reachability analyses to high-dimensional problems. However, effectively learning and analyzing the NLF is still an open question. In this paper, we firstly propose a co-learning algorithm to build the controller and the NLF in one training loop, which benefits both the controller and NLF. Furthermore, we analyze the NLF with a computationally effective approach. This allows our algorithm to be effectively deployed to realize safety-sensitive navigation tasks. 

\begin{figure}[t]
    \centering
    \includegraphics[width= \linewidth]{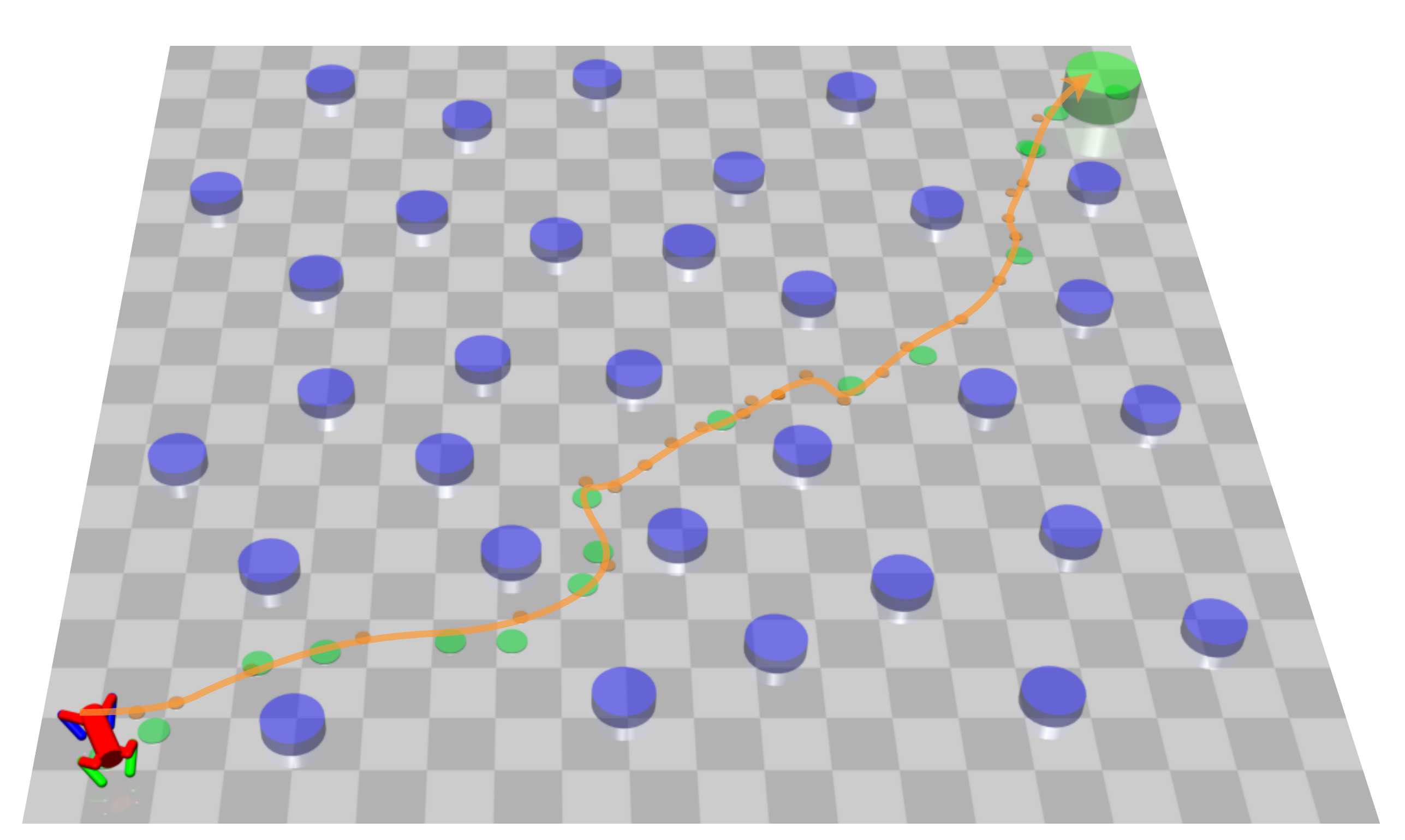}
    \caption{Level-2 Quadruped Navigation Task. Blue cylinders represent hazardous zones, and the small green dots represent the waypoints generated by RRT. The final goal is the large green cylinder. We navigate the quadruped robot to achieve the goal while avoiding any hazardous zones. }
    \label{fig:quadruped-lv-2}
\end{figure}

We evaluate our approach on the Safety Gym suite~\cite{Ray2019} in highly cluttered environments with three levels of obstacles complexity and each with four complex dynamical systems, named \emph{sweeping} (similar to rumba), \emph{point} (a circular rigid-body), \emph{car}, and a 58 DOF \emph{quadruped} (Fig.~\ref{fig:quadruped-lv-2}). These environments are challenging for traditional DRL methods even without safety constraints. However, this paper demonstrates that our approach can scale well to these high-dimensional, cluttered navigation tasks while explicitly incorporating safety constraints. Our framework effectively co-learns a TNLF in the DRL loop and leverages the learned TNLF in a computationally tractable manner. 


Our main contributions are as follows:
\begin{enumerate}
    \item We propose a model-free Lyapunov method that can provide significant safety enhancement for high-dimensional safety-sensitive robot navigation problems with raw sensor observations.
    \item We co-learn a TNLF and controller in a manner that improves controller performance and yields a high-quality Lyapunov function.
    \item We use the learned TNLF to build a computationally effective runtime monitor for heuristically guiding robots under safety constraints at runtime.
    \item We demonstrate that the combined planner and NLF can consistently reach the goal state in challenging safety-sensitive navigation tasks while also providing significantly fewer safety violations and a higher reach rate than modern constrained RL methods.
\end{enumerate}

\section{Background}

\subsection{Reinforcement Learning}
\label{sec:rl}
Reinforcement learning (RL) generates an optimal controller by interacting with an environment given a scalar reward signal.  DRL scales RL to problems with high dimension state and action spaces by means of neural networks \cite{lillicrap2015continuous}. Most DRL algorithms consider problems under the Markov Decision Process (MDP) setting. Given a state space $\mathcal{S}$ and an action space $\mathcal{A}$, the MDP models interaction with an environment using a transition function $\mathcal{T}(s_t, a_t, s_{t+1}): \mathcal{S} \times \mathcal{A} \times \mathcal{S} \rightarrow \mathcal{P}$. The function $\mathcal{T}$ returns the probability of transitioning from a state $s_t$ to a state $s_{t+1}$ after taking an action $a_t$. Meanwhile, a function $R(s_t, a_t, s_{t+1})$ measures the transition reward. Given a discount factor $\gamma \in [0, 1]$, a controller $\pi: \mathcal{S} \rightarrow \mathcal{A}$ learns to maximize the discounted cumulative reward function $\sum_{t=0}^{T-1} \gamma^t R(s_t, a_t, s_{t+1})$ where $T$ is the total time horizon.

Many model-free on-policy (\cite{trpo,ppo}) and off-policy (\cite{ddpg,td3,sac}) DRL algorithms appeared in recent years. Among all these algorithms, our work is closely related to the DDPG \cite{lillicrap2015continuous}. An important component of DDPG is the $Q$ function $Q: \mathcal{S} \times \mathcal{A} \rightarrow \mathbb{R}$. $Q(s_t, a_t)$ returns the discounted cumulative reward after taking action $a_t$ under state $s_t$. A higher return of $Q$ means a better action. Thus knowing $Q(s_t, a_t)$ allows us to choose the best action. In DDPG, a learned Q function serves as the critic for actions. An actor policy $\pi(s_t)$ is trained to maximize $Q(s_t, \pi(s_t))$. From the Q function we can derive the optimal policy being $\pi^*(s_t) = \underset{\pi}{\arg \max}\ Q(s_t, \pi(s_t))$.

\subsection{Lyapunov Method}

\subsubsection{Lyapunov Function}
\label{sec:lyapunov_function}
A Lyapunov function characterizes the stability property of a controller. It is a function satisfying the following constraints:
\begin{align}
     & V(s_o)                                    & = 0 & \label{eq:origin}         \\
     & V(s_t)              & > 0 &, \forall s_t \neq s_o\label{eq:positive}       \\
     & V\left(s_{t+1}\right) - V\left(s_t\right) & < 0 & \label{eq:lie_derivative}
\end{align}

In Eq.~\eqref{eq:origin}, a Lyapunov function's value is 0 at the origin $s_o$. $s_o$ is a stabilized state of the controller. Eq.~\eqref{eq:positive} enforces that the Lyapunov function is always positive when its input is not $s_o$. The L.H.S. of Eq.~\eqref{eq:lie_derivative} is known as the lie derivative, $\nabla_\pi V = V\left(s_{t+1}\right) - V\left(s_t\right)$. When the lie derivative is smaller than 0, $V(s_t)$ strictly decreases over time.  The lie derivative depends on controller $\pi$, since computing $s_{t+1}$ requires action $a_t = \pi(s_t)$. 

\subsubsection{Neural Lyapunov Function}
\label{sec:neural_lyapunov_function}
A neural Lyapunov function $V(s_t): \mathcal{S} \rightarrow \mathbb{R}$ models a Lyapunov function via a neural network satisfying Eqs.~\eqref{eq:origin}\eqref{eq:positive}\eqref{eq:lie_derivative}. Training the NLF requires minimizing the Lyapunov risk \cite{richards2018lyapunov,chang2020neural}. The Lyapunov risk is defined as
\begin{equation}
    \label{eq:lyapunov_risk}
    \begin{aligned}
        L_{lf}(\theta_{V})= \mathbb{E}_{s_t \sim (E, \pi)}[ & V_{\theta_{V}}^{2}(s_o)                                \\
                                                            & + \max \left(0,-V_{\theta_{V}}(s_t)\right)             \\
                                                            & + \max \left(0, \nabla_\pi V_{\theta_{V}}(s_t)\right)]
    \end{aligned}
\end{equation}
Here, $\theta_{V}$ represents the parameters of the NLF with $\pi$ being a controller. The Lyapunov risk is an expectation over $s_t$ sampled from environment $E$ and controller $\pi$. Minimizing $V_{\theta_{V}}^2(s_o)$ allow $V_{\theta_{V}}$ to satisfy Eq.~\eqref{eq:origin}, while minimizing $\max \left(0,-V_{\theta_{V}}(s_t)\right)$ and $\max \left(0, \nabla_\pi V_{\theta_{V}}(s_t)\right)$ makes $V_{\theta_{V}}$ satisfy Eq.~\eqref{eq:positive} and Eq.~\eqref{eq:lie_derivative} respectively. $\nabla_\pi V_{\theta_{V}}$ is the lie derivative over controller $\pi$ as mentioned above.

\subsubsection{Region of Attraction}
The Lyapunov function specifies a \emph{Region of Attraction} (RoA) as
\begin{align}
    \label{eq:roa}
    \text{RoA} = \{ s | V(s) < C_{RoA}\},
\end{align}
where $C_{RoA} \in \mathbb{R}^+$ is a constant. Since the Lyapunov function strictly decreases over time, if we initialize a robot with any state in an RoA, the robot will always stay in the RoA in future. Since $V(s_o) = 0$, the origin $s_o$ must be included in the RoA. It is also known as the sink of the RoA.

\subsection{Path Planning}

Sample-based planning methods like RRT \cite{lavalle2006planning} can be used to find a path to reach a goal state. RRT grows a collision-free tree from a given start state by randomly sampling the robot's configuration space. Once the tree finds a goal state, a Dijkstra algorithm extracts the shortest path connecting the given start and goal states for our navigation tasks.

\subsection{Goal-Conditioned State Space}
\label{sec:goal_conditioned_state_space}

Goal-conditioned RL (\cite{her2017}) augments the state space by conditioning it with a static goal $g \in \mathcal{G}$ where $\mathcal{G}$ is the goal space.  In our problem, we consider a 2D goal space ($\mathcal{G} \subseteq \mathbb{R}^2$) representing a position in the 2D plane where a robot operates.  Suppose the position of a robot under state $s$ is $pos(s) \in \mathbb{R}^2$, we define a goal vector
\begin{align*}
    d_g (s) = g - pos(s)
\end{align*}
The state $s_g \in \mathcal{S}_g$ is a modified version of $s\in \mathcal{S}$ to contain this goal vector $d_g(s)$, and the intrinsic state $s_{/g}$ of a robot. The intrinsic state does not change as the robot position changes. Thus, every goal and state pair $(g, s)\in \mathcal{G} \times \mathcal{S}$ is mapped to
\begin{align}
    \label{eq:state_space}
    s_g(s) = [d_g (s), s_{/g}]
\end{align}
We train the controller and NLF in this goal-conditioned space to provide the flexibility of generalizing the learning-based components to different goals.

Since we must generalize the Lyapunov function to different goals as well, we always set the sink position $pos(s_o)$ to the current goal $g$.  This effectively translates the Lyapunov function to be centered around an arbitrary goal $g$ which means $pos(s_o) = g$ and the goal vector $d_g(s_o) = [0,0]$ (when $\mathcal{G} \subseteq \mathbb{R}^2$). We choose this setting because our controller's objective is to arrive and be stable at the goal, while the Lyapunov function desires that the robot is asymptotically stabilized to $s_o$.

\section{Approach}
Fig.~\ref{fig:approach_overview} depicts an overview of our approach. The navigation task requires guiding a robot to reach the final goal while avoiding collision with hazardous zones.

\begin{figure}[ht]
    \includegraphics[width=\linewidth]{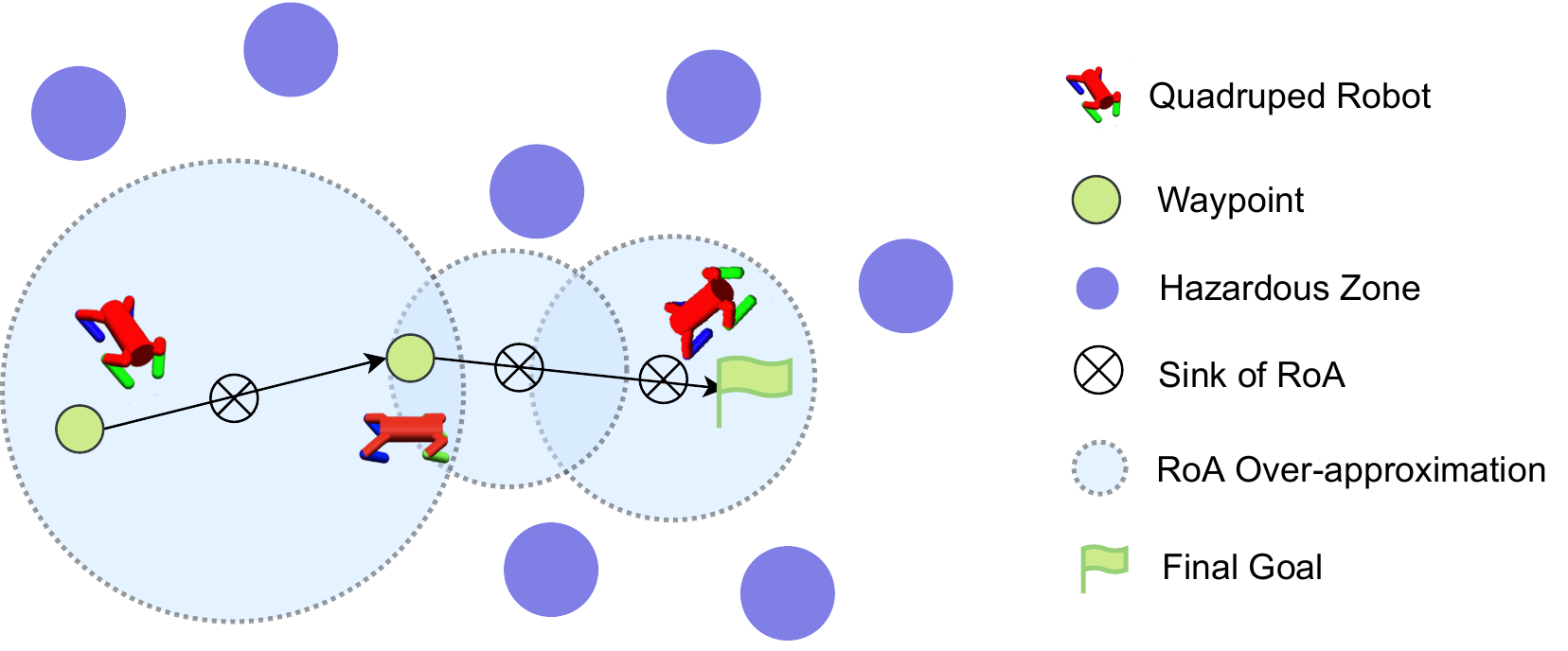}
    \caption{Approach Overview: A robot is safely navigated from start to goal in over-approximated regions of attraction.}
    \label{fig:approach_overview}
\end{figure}

Our approach has three steps. First, we co-learn the controller and TNLF in a DRL loop. The controller learns to reach a given target in the shortest time. The design of the TNLF is introduced in Sec.~\ref{sec:TNLF}, and details about the co-learning procedure are given in Sec.~\ref{sec:co-learning-approach}. Second, we run RRT on the 2D plane to build a collision-free path between the initial and goal positions. The path is a sequence of waypoints. Third, we use the learned controller to follow waypoints generated by the RRT planner and provide the safety guarantee with a runtime monitor. This runtime monitor places and builds the RoAs over-approximations with the learned TNLF, and it is introduced in Sec.~\ref{sec:runtime_monitor}.


\subsection{Twin Neural Lyapunov Function}
\label{sec:TNLF}

The TNLF is a key component of our co-learning algorithm. Firstly, we define a function \lqffn{s_t, a_t} to integrate the Lyapunov function into the DDPG training loop. \lqffn{s_t, a_t} predicts the Lyapunov function value $V(s_{t+1})$ after taking action $a_t$. Then, similar to DDPG introduced in Sec.~\ref{sec:rl}, we can train our policy to minimize \lqffn{s_t, \pi(s_t)}. Since $\lqf$ works similarly to the $Q$ function in DDPG, we name \lqffn{s_t, a_t} the \emph{Lyapunov Q function} and train \lqffn{s_t, a_t} with a regression loss,
\begin{align}
    \label{eq:lqf_loss}
    L_{lqf}(\theta_{\lqf}) = || \lqf_{\theta_{\lqf}}(s_t, a_t) - V_{\theta_{V}}(s_{t+1}) ||_2
\end{align}
where $s_{t+1}$ is the state that results from taking  action $a_t$ in $s_t$. Eq.~\eqref{eq:lqf_loss} is only used to update $\theta_\lqf$. The parameter $\theta_{V}$ is fixed when updating parameters with Eq.~\eqref{eq:lqf_loss}. Finally, we call $V(s_t)$ and \lqffn{s_t, a_t} together as the twin neural Lyapunov function.

\subsection{Co-learn the TNLF with controller}
\label{sec:co-learning-approach}

Co-learning the TNLF with the controller has significant benefits. First, when the Lyapunov function learns to characterize a controller, the controller is also adapted to the Lyapunov function. Thus, we can learn a better NLF to characterize stability properties. Second, the integration of the Lyapunov function can accelerate and stabilize controller convergence. This is because the Lyapunov function provides an additional training signal for RL. Similar ideas have also appeared in \cite{td3, sac}, where the authors employ double $Q$ functions for a better training signal.

\begin{algorithm}[htp]
    \caption{Co-learn TNLF with Control Policy}

    \label{algo:co-learn}
    \Notions{Environment $E$, Policy $\pi$, Q function $Q$, Lyapunov Function $\color{red}V$, Lyapunov Q function $\color{red}\lqf$, Training Network Parameter $\theta_{[\cdot]}$, Target Network Parameter $\theta'_{[\cdot]}$, Training Batch Size $N$, Polyak Constant $\tau$}

    Initialize training network $\pi_{\theta_\pi}, Q_{\theta_Q}$, $\color{red}V_{\theta_{V}}$, $\color{red}\lqf_{\theta_{\lqf}}$;

    Create target network $\pi_{\theta'_\pi}, Q_{\theta'_Q}$,$ \color{red}\lqf_{\theta'_{\lqf}}$;

    Initialize all target networks parameters $\theta'_{[\cdot]} \leftarrow \theta_{[\cdot]}$;

    \For{$i = 1, \dots, T_{ep}$}{
        Replay Buffer $\mathcal{R} \leftarrow \mathtt{Transitions}(E, \pi_{\theta'_\pi})$;\\
        \For{$j = 1, \dots, T_{grad}$}{
            Data Batch $\mathcal{D} \leftarrow \mathtt{Sample}(\mathcal{R}, N)$;\\
            $\mathtt{TrainQ}(Q_{\theta_Q} \mid \mathcal{D})$;\\
            $\color{red} \mathtt{TrainTNLF}(V_{\theta_V}, \lqf_{\theta_{\lqf}} \mid \mathcal{D})$; \\
            $\color{red} \mathtt{TrainControlPolicy}(\pi_{\theta_{\pi}} \mid Q_{\theta'_Q}, \lqf_{\theta'_{\lqf}}, \mathcal{D})$;\\
            $\mathtt{PolyakUpdate}(\theta'_{[\cdot]} \mid \theta_{[\cdot]}, \tau)$;
        }
    }
\end{algorithm}

Algorithm~\ref{algo:co-learn} describes our co-learning framework. The policy network $\pi_{\theta_\pi}$, the Q function network $Q_{\theta_Q}$, and their target networks $\pi_{\theta'_\pi}$ and $Q_{\theta'_Q}$ come from the original DDPG algorithm. We integrate our Lyapunov function network $V_{\theta_V}$ and Lyapunov Q function network $\lqf_{\theta_{\lqf}}$ into the co-learning loop with the functions and variables highlighted in \textcolor{red}{red}. The Lyapunov function network $V_{\theta_V}$ will be used in our downstream planning algorithm. The Lyapunov Q function network \lqffn{s_t, a_t} serves as another critic for action $a_t$. As a critic, the \lqffn{s_t, a_t} affects the stability of the whole training process. Thus, we also created a target network $\lqf_{\theta'_{\lqf}}$ to avoid it changing dramatically.

Line 1 to 3 of Algorithm~\ref{algo:co-learn} initialize all the networks' parameters. Line 5 samples the transitions with target policy network $\pi_{\theta'}$ and store these transitions to replay buffer $\mathcal{R}$. The loop starting from line 5 updates the parameters of each network. Line 7 samples training data batch $\mathcal{D}$ from replay buffer $\mathcal{R}$. Line 8 updates $\theta_Q$ with sampled data $\mathcal{D}$ with the approach in \cite{lillicrap2015continuous}. Line 9 trains the $V_{\theta_V}, \lqf_{\theta_{\lqf}}$ with $\mathcal{D}$. The loss function for $V_{\theta_V}$ is the Lyapunov risk defined in Eq.~\eqref{eq:lyapunov_risk}, and $\lqf_{\theta_{\lqf}}$ is trained with loss in Eq.~\eqref{eq:lqf_loss}. Line 10 trains controller $\pi_{\theta_{\pi}}$ with $Q_{\theta'_Q}, \lqf_{\theta'_{\lqf}}, \mathcal{D}$. The controller $\pi_{\theta_\pi}$ is trained with the loss
\begin{align}
    L_{\pi_{\theta_\pi}} = \mathbb{E}_{s_t}\ [-Q_{\theta'_Q}(s_t, \pi_{\theta_\pi}(s_t)) + \alpha \lqf_{\theta'_{\lqf}}(s_t, \pi_{\theta_\pi}(s_t))]
\end{align}

The training goal is to minimize $L_{\pi_{\theta_\pi}}$ via optimizing policy $\pi_{\theta_\pi}$. Minimizing the $-Q_{\theta'_Q}$ term leads to maximizing the cumulative reward predicted by $Q_{\theta'_Q}$. When minimizing  $\lqf_{\theta'_{\lqf}}$, the controller is learning to minimize the value of the Lyapunov function in the next step, which adapts the controller to satisfy Eq.~\ref{eq:lie_derivative}. Here, $\alpha$ is a hyperparameter that controls the $\lqf$ impact on the update of policy $\pi$. Line 11 conducts a Polyak update on the target networks for enhanced learning stability \cite{lillicrap2015continuous}. It updates all the target parameters with
\begin{align*}
    \theta'_{[\cdot]} = (1-\tau)\theta'_{[\cdot]} + \tau \theta_{[\cdot]}
\end{align*}
where $\tau \in (0,1]$, and $[\cdot]$ can be $\pi, \mathcal{Q}$ or $Q$.

\subsection{Runtime Monitor}
\label{sec:runtime_monitor}

Algorithm~\ref{algo:co-learn} provides us with a controller $\pi$ and an NLF $V$ characterizing $\pi$.  We then run the RRT planner to generate a collision-free path to reach the final goal. A robot follows this path with the controller $\pi$, while its safety is ensured by a sequence of RoAs placed over the path. How to build these RoAs and where to place them are the two problems we address.

\begin{figure}[htb]
    \centering
    \begin{subfigure}[b]{0.47\linewidth}
        \includegraphics[width=\linewidth]{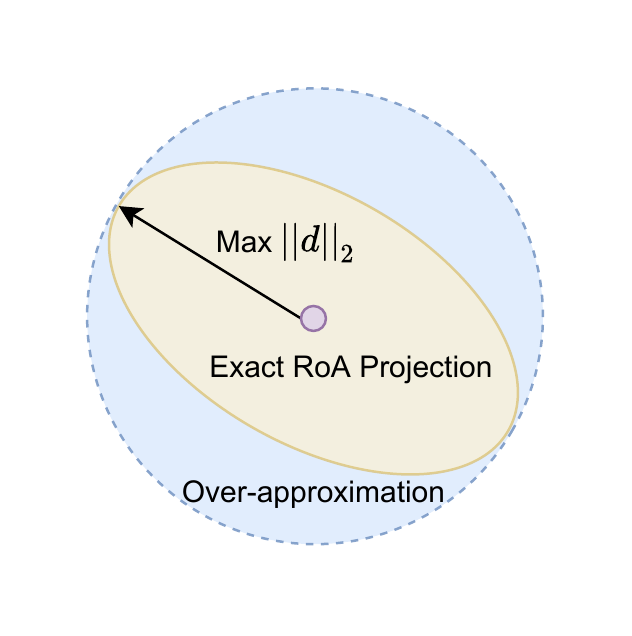}
        \vspace{-27px}
        \caption{RoA Over-approximation}
        \label{fig:roa_over_appr}
    \end{subfigure}
    \begin{subfigure}[b]{0.39\linewidth}
        \includegraphics[width=\linewidth]{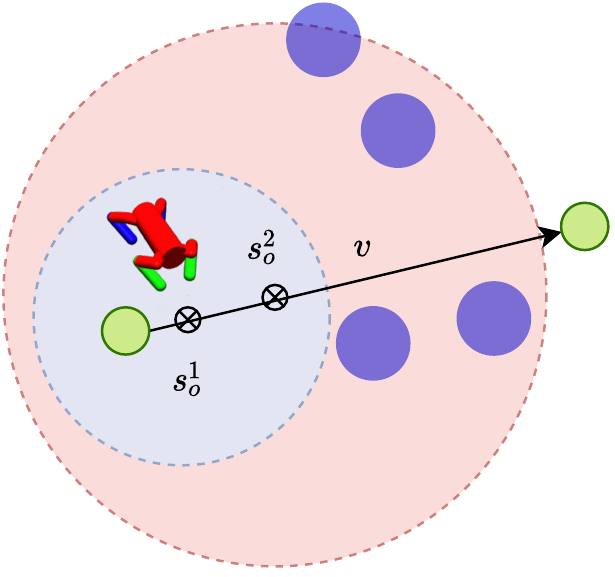}
        \caption{Search Sink}
        \label{fig:search_sink}
    \end{subfigure}
    \caption{Our runtime monitor builds and aligns over-approximated RoAs around planner's waypoints.}
\vspace{-0.5cm}\end{figure}

\paragraph{Building  RoA over-approximations}
The runtime monitor builds the RoA over-approximation for the exact RoA projection on the 2D path plane. Given an NLF $V$, we can build an RoA in the state space $\mathcal{S}$. Because we only care about the collision in the 2D path plane, we project the RoA to the path plane. If this projection does not collide with a hazardous zone, the RoA inside must not collide with this hazardous zone. However, computing the exact projection can be hard. Instead, we propose an approach to compute the over-approximation of this projection with a demonstration provided in Fig.~\ref{fig:roa_over_appr}. Given a NLF $V$ and a constant $C_{RoA}$, the RoA is specified as $\{s \mid V(s) < C_{RoA}\}$. The over-approximation we build is a circle around the projected RoA in the 2D path plane. If we place the center of this circle in the position of the sink (the current robot goal $g$), the radius we need to compute is the max L2-norm of the goal vector $d_g(s)$ for all the states in this RoA. We search the max $||d_g(s)||_2$ by maximizing the objective function in Eq.~\eqref{eq:over_approx}. 
\begin{align}
    \label{eq:over_approx}
    L_{ap}(s) = ||d_g(s)||_2 + \beta \min(C_{RoA} - V(s), 0)
\end{align}
When $\beta\in \mathbb{R^+}$ is a large positive constant, $\beta \min(C_{RoA} - V(s), 0)$ is a constraint which forces the search to stay in the RoA $\{s \mid V(s) < C_{RoA}\}$. When  $s$ is in the RoA specified by $V$ and $s_o$, then $C_{RoA} \geq V(s)$ and $\min(C_{RoA} - V(s), 0)$ will be 0. Otherwise, $\min(C_{RoA} - V(s), 0)$ will be negative and penalizes the objective function. We optimize
\begin{align}
    \label{eq:over_apprx_opt}
    s^* = \underset{s \in \mathcal{S}}{\arg \max}\ L_{ap}(s)
\end{align}
where we choose a large $\beta$ to enforce $\min(C_{RoA} - V(s), 0) = 0$. We sample a batch of initial optimization states from $\mathcal{S}$ and optimize them with projected gradient descent~\cite{chen2015fast} to estimate a solution to Eq.~\eqref{eq:over_apprx_opt}. Since $s^* = [d_g(s^*), s^*_{/g}]$, we can compute the radius as $||d_g(s^*)||_2$.

\paragraph{Placing RoA over-approximation}
Fig.~\ref{fig:search_sink} shows how the runtime monitor places the sink and RoA over-approximations along the planned path. We prefer a larger RoA over-approximation because smaller regions usually result in more conservative behavior and make the robot move slower. However, if these RoA over-approximations intersect with the hazardous zones, the robot may violate safety properties. Therefore, our goal is to find the largest over-approximation that has no intersection with the hazardous zones. Given a vector $v = g_2 - g_1 \in \mathbb{R}^2$ from one waypoint $g_1 \in \mathcal{G}$ to the next waypoint $g_2 \in \mathcal{G}$, we search the positions of sinks between $g_1$ and $g_2$ using a line search. The position of a sink is given by $$pos(s_o^i) = g_1 + i \delta v $$ where $i \in \mathbb{N}^+$, $\delta$ is a number controlling the granularity of the search, $\delta \in (0, 1]$. Given a state $s_t$, we build the smallest RoA, denoted as $\text{RoA}^{*}$, which includes $s_t$.
\begin{align}
    \label{eq:smallest_roa}
    \text{RoA}^{*}(s_t) = \{s \mid V(s) \leq V(s_t)\}
\end{align}
The state $s_t$ depends on the position of sink $pos(s_o^i)$ (i.e. goal of a goal-conditioned state space).  Hence, choosing different sink positions results in different RoA over-approximation as demonstrated in Fig.~\ref{fig:search_sink}. To find the largest over-approximation that has no intersection with hazardous zones, we need to compute the radius of over-approximation for every sink position via optimizing Eq.~\ref{eq:over_approx}. Finally, we select the sink position with the largest radius and set it as the goal of the robot. We repeat this line searching in every step. Intuitively, this makes the RoA over-approximation change adaptively based on surrounding hazardous zones.

\paragraph{Pre-computed RoA-overapproximations}
Because computing the over-approximation in real-time can be computationally expensive, in practice, we pre-compute and reference the radius of over-approximation via a lookup table. Eq.\eqref{eq:roa} tells us that the shape of the RoA only depends on $C_{RoA}$ when given a trained NLF $V$. Thus, one $C_{RoA}$ can only correspond to one minimal circle over-approximation. 
We pre-compute a lookup table that maps the $C_{RoA}$ to its corresponding radius of the minimal circle over-approximation. According to Eq.~\eqref{eq:smallest_roa}, the $C_{RoA}$ of $\text{RoA}^{*}(s_t)$ is determined by $V(s_t)$. However, because the lookup table is finite and cannot capture all possible $C_{RoA}$s, we over-approximate the RoA with a larger circle, instead of the minimal one.

\begin{theorem}
    \label{trm:lookup_table}
    Suppose the lookup table has keys $C_{RoA}^1, \cdots C_{RoA}^N$ in increasing order, given a state $s_t$, and $C_{RoA}^i < V(s_t) \leq C_{RoA}^{i+1}$, the lookup table returns the radius corresponding to $C_{RoA}^{i+1}$. The over-approximation built with the returned radius must over-approximate $\text{RoA}^{*}(s_t)$.
\end{theorem}

\begin{proof}
    \begin{align*}
        \text{RoA}^{*}(s_t) & = \{s \mid V(s) \leq V(s_t)\}        \\
        \overline{\text{RoA}}  & = \{s \mid V(s) \leq C_{RoA}^{i+1}\}
    \end{align*}
    $V(s_t) \leq C_{RoA}^{i+1} \implies \text{RoA}^{*}(s_t) \subseteq \overline{\text{RoA}}$.  Suppose a point $p \in \text{RoA}^{*}(s_t)$ has the largest L2-norm to the sink $s_o$. $||p - pos(s_o)||_2$ is the over-approximation radius of $\text{RoA}^{*}(s_t)$. $\text{RoA}^{*}(s_t) \subseteq \overline{\text{RoA}} \implies p \in \overline{\text{RoA}}$. Thus, the over-approximation radius of $\overline{\text{RoA}}$ is at least $||p - pos(s_o)||_2$. Hence, the over-approximation of $\overline{\text{RoA}}$ must over-approximates $\text{RoA}^{*}(s_t)$.
\end{proof}

According to Theorem~\ref{trm:lookup_table}, the lookup table can always return the radius corresponding to $C_{RoA}^{i+1}$ for building a circle over-approximation of $\text{RoA}^{*}(s_t)$. The pre-computed over-approximation can boost computational speed. Over-approximating an RoA will only require one forward computing on the neural network model $V$ and a query in the lookup table. In this way, our algorithm can be effectively deployed online.

\section{Experiments}

Our experiments aim to answer the following questions:
\begin{itemize}
  \item Is co-learning a viable way to learn a performant policy while gaining a high-quality Lyapunov function?
  \item Can we leverage the learned Lyapunov function for more safety-oriented tasks?
  \item Does our safety-oriented framework sacrifice the completion of the objective and the speed of reaching the goal in favor of safety?
\end{itemize}

\subsection{Setup}
\label{sec:exp_setup}

We show results on both a custom 2D sweeping robot environment as well as robotic environments in Safety Gym \cite{Ray2019}  with continuous state and action spaces.

\begin{figure}[ht]
  \centering
  \begin{subfigure}[b]{0.4935\linewidth}
    \includegraphics[width=\textwidth]{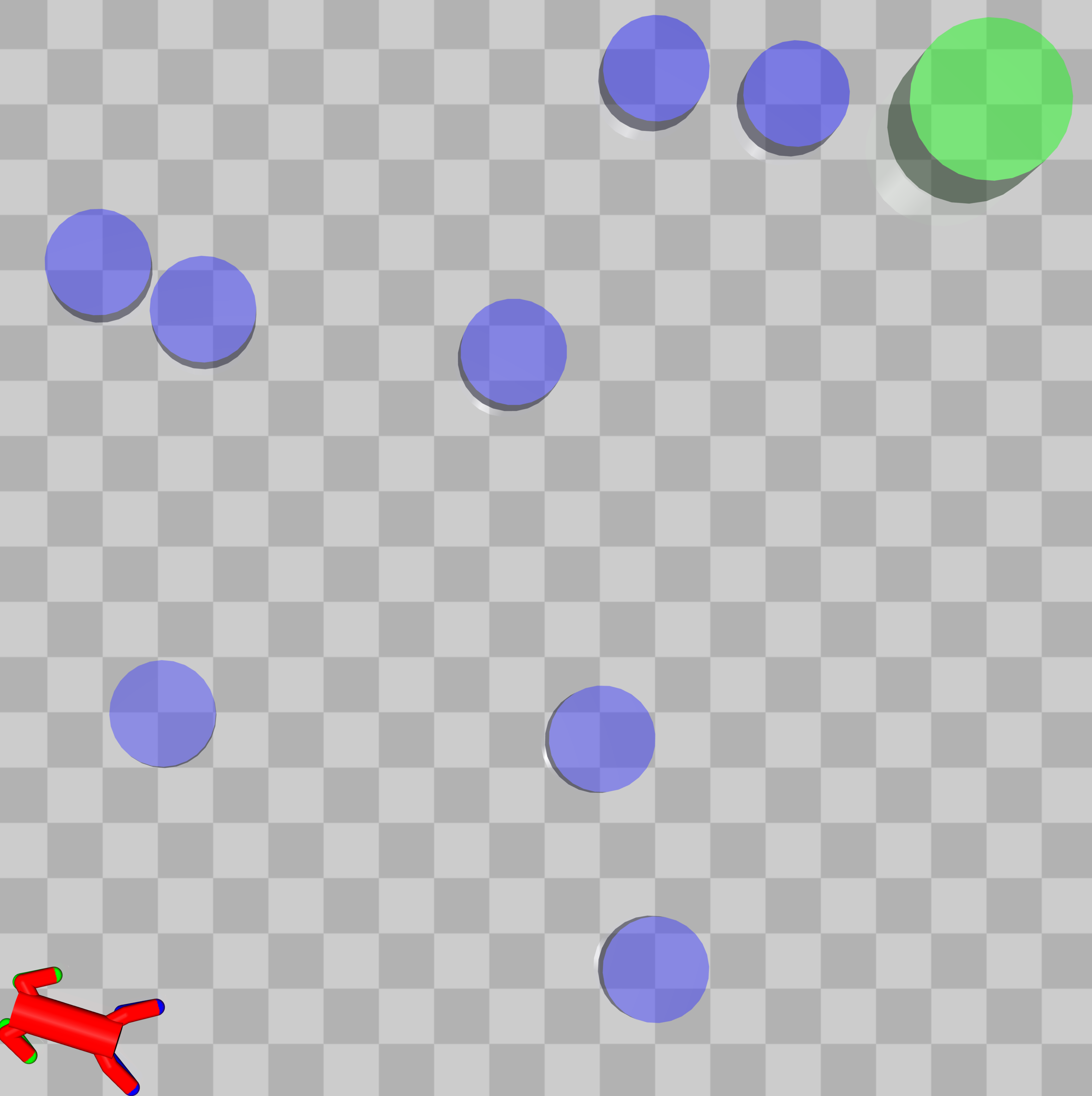}
    \caption{Level-1 Map}
  \end{subfigure}
  \begin{subfigure}[b]{0.49\linewidth}
    \includegraphics[width=\textwidth]{images/quadruped-lv-3.pdf}
    \caption{Level-3 Map}
  \end{subfigure}
  \caption{Difficulty Levels of Navigation Tasks}
  \label{fig:difficulty_of_tasks}
\end{figure}

The benchmarks for each robot were classified into three levels of difficulty based on the number of hazardous zones and map size as shown in Fig.\ref{fig:quadruped-lv-2} and Fig.\ref{fig:difficulty_of_tasks}. The initial and goal positions are placed on the map's lower-left corner and upper-right corner, respectively. The hazardous zone positions are initialized randomly for each run. The map size for difficulty level-1 to level-3 are $4 \times 4, 8 \times 8$ and $16 \times 16$, respectively. Level-1 has $8$ hazardous zones, and level-2 and level-3 have $32$ and $128$ hazardous zones, respectively.


\begin{figure}[htb]
  \begin{subfigure}[b]{0.23\linewidth}
    \includegraphics[width=0.8\textwidth]{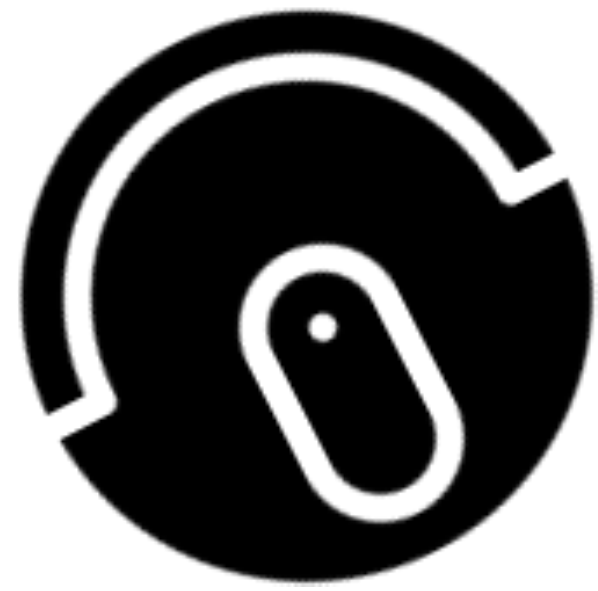}
    \caption{\textit{Sweeping}}
  \end{subfigure}
  \begin{subfigure}[b]{0.23\linewidth}
    \includegraphics[width=0.8\textwidth]{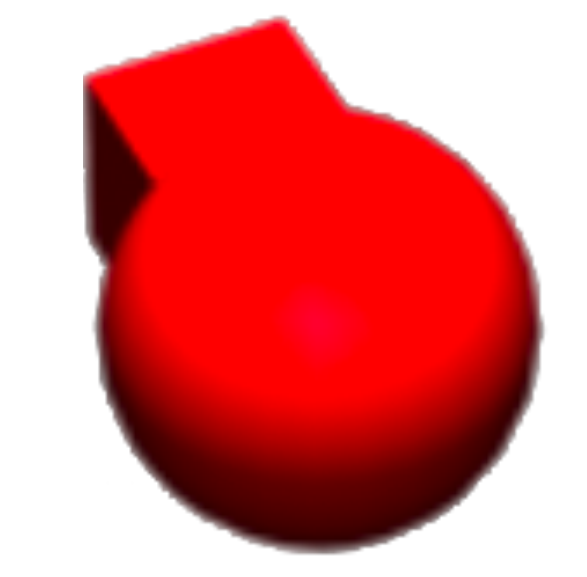}
    \caption{\textit{Point}}
  \end{subfigure}
  \begin{subfigure}[b]{0.23\linewidth}
    \includegraphics[width=\textwidth]{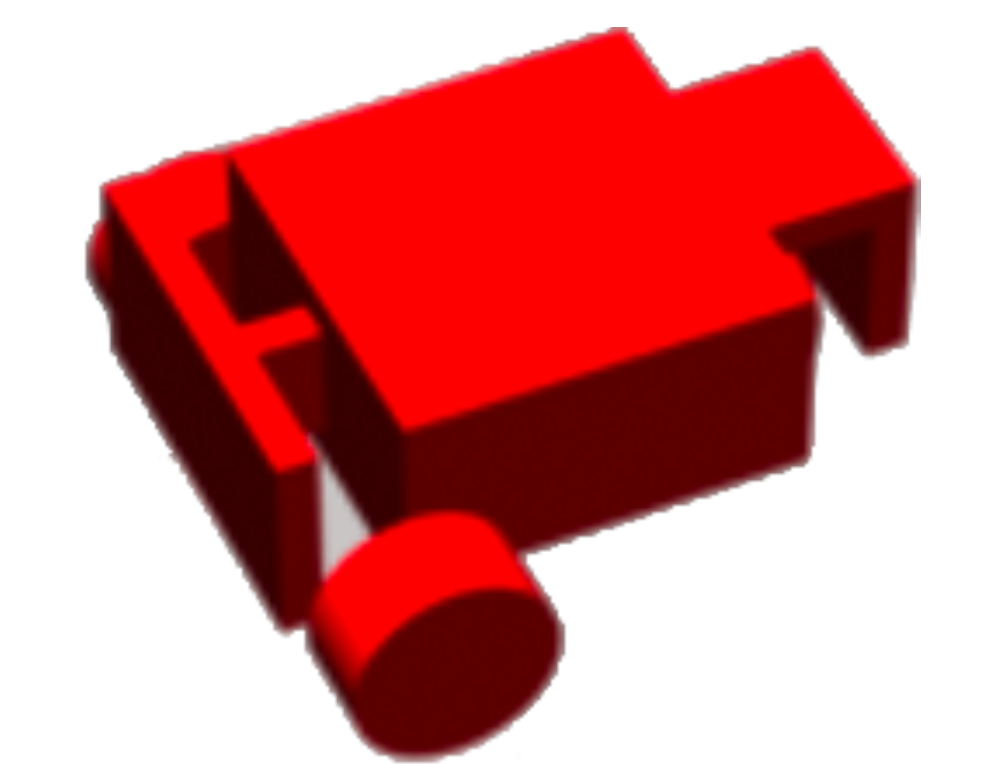}
    \caption{\textit{Car}}
  \end{subfigure}
  \begin{subfigure}[b]{0.23\linewidth}
    \includegraphics[width=\textwidth]{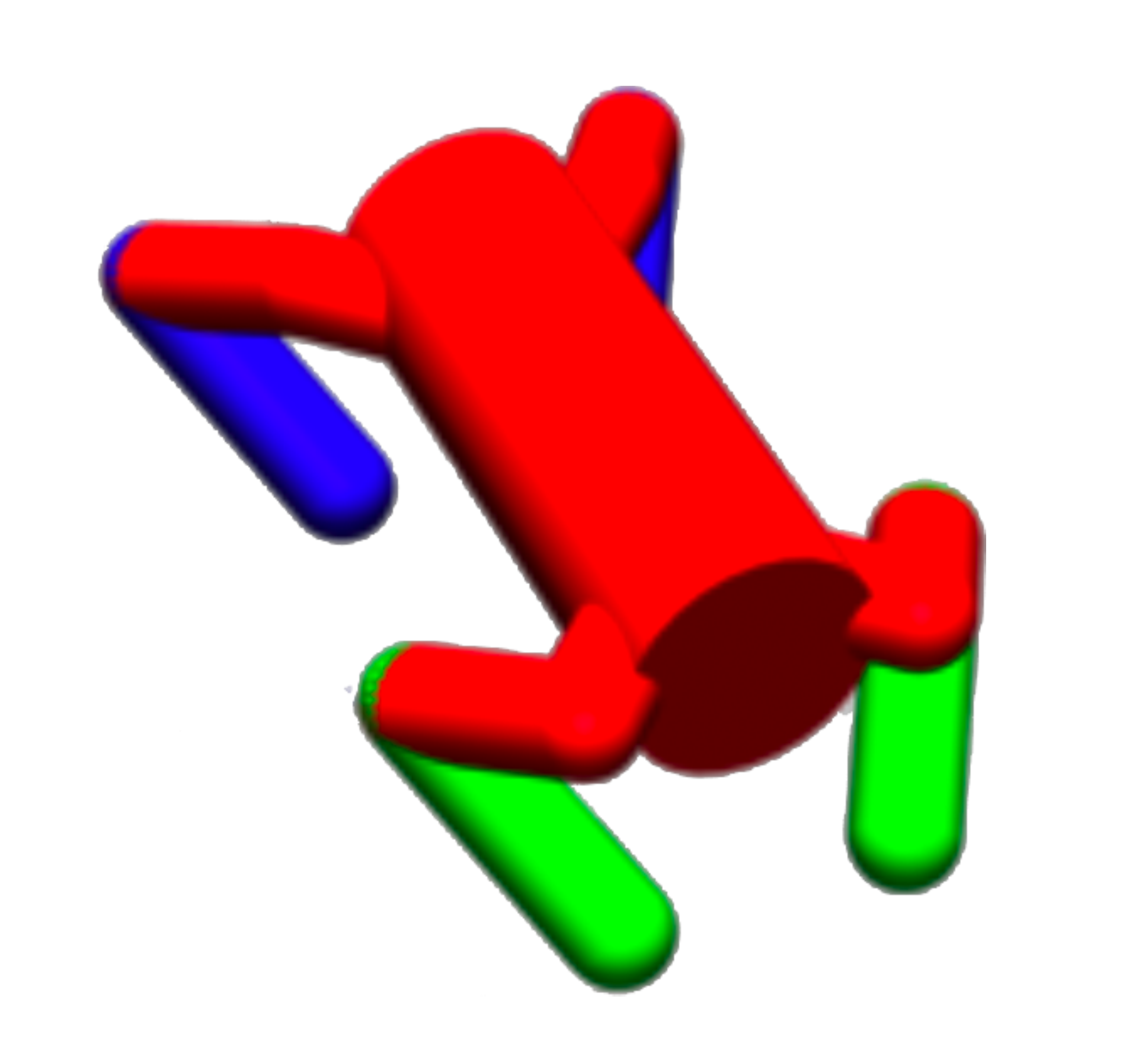}
    \caption{\textit{Quadruped}}
  \end{subfigure}
\caption{Our dynamical systems include \textit{sweeping}, \textit{point}, \textit{car}, and \textit{quadruped}. Their corresponding state and action space dimensions, denoted as (state dim, action dim), from (a)-(d) are (2,2), (14,2), (26,2) and (58,12), respectively.}
\label{fig:dynamical_models}
\end{figure}

Fig.\ref{fig:dynamical_models} depicts our dynamical models with their state and action space dimensions. \textit{Sweeping} is a customized environment for moving the position of a sweeping robot in a 2-D plane. \textit{Point} models a robot restricted to the 2-D plane with actuators to rotate and move forward. \textit{Car} has two independently operated parallel wheels with a rear-wheel for balance. \textit{Quadruped} models a quadrupedal robot with each leg having torque controls in the hip and knee joints. All agents obtain their state information from the joints, accelerometer, gyroscope, magnetometer, velocimeter, and a 2D vector toward the goal.

\subsection{Co-learning}

The low level controller and TLCF are trained in an environment without hazardous zones. Fig. \ref{fig:training_perf} shows the learning progress for each of our robots in this hazardous zone-free environment. The reward used for training the controller is defined as
\begin{align}
  \label{eq:controller_reward}
  r_t = ||g - pos(s_t)||_2 - ||g - pos(s_{t+1})||_2
\end{align}
where $r_t$ is the reward at time $t$, $g$ is the goal position the controller should achieve, $pos(s_t)$ and $pos(s_{t+1})$ are the positions of robot at time $t$ and $t+1$, respectively. This reward is positive when $pos(s_{t+1})$ is closer to goal than $pos(s_t)$.

\begin{figure}[htb]
  \includegraphics[width=\linewidth]{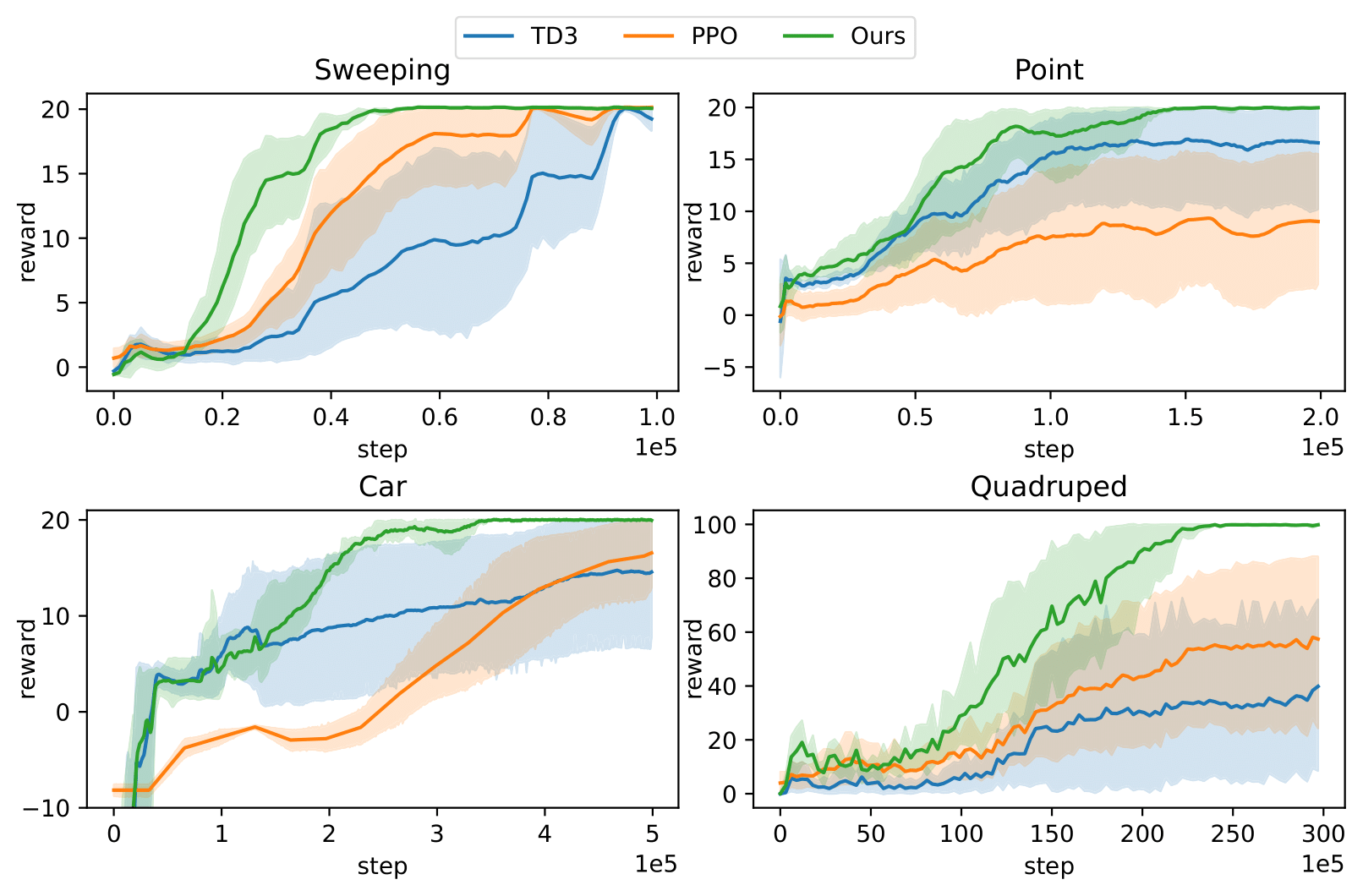}
  \caption{Training performance depicting total accumulated rewards over number of interaction steps.}
  \label{fig:training_perf}
  \vspace{-5px}
\end{figure}

Fig. \ref{fig:training_perf} compares the training reward between our co-learning algorithm with the off-policy TD3~\cite{td3} and on-policy PPO~\cite{ppo}. Our co-learning algorithm can reach higher reward in fewer simulation steps, while it is also more stable after the reward converges on all four robots. The experiments show that the training signal from the Lyapunov Q function critic can benefit the controller's training process.

\begin{table}[!htp]\centering
  \caption{ Eq.~\eqref{eq:positive}~\eqref{eq:lie_derivative} Satisfaction Rate of NLF}\label{tab:nlf_perf}
  \begin{tabular}{c|cccc}\toprule
    Phase                                              & \textit{Sweeping} & \textit{Point}    & \textit{Car}      & \textit{Quadruped} \\\midrule
    Co-learn                                           & \textbf{99.67 \%} & \textbf{98.97 \%} & \textbf{97.06 \%} & \textbf{97.59 \%}  \\
    Post-learn~\cite{chang2020neural, xiong2022hisarl} & 99.42\%           & 96.77 \%          & 94.54\%           & 93.47 \%           \\
    \bottomrule
  \end{tabular}
\end{table}

The quality of NLF is measured with Eq.~\eqref{eq:origin}~\eqref{eq:positive}~\eqref{eq:lie_derivative}. For Eq.~\eqref{eq:origin}, all the sinks have $|V(s_o)| < 10^{-3}$ on the four robots. We generated 100,000 sampled state transitions $(s_t, s_{t+1})$ for each robot. These transitions are generated with the trained controller $\pi$. We evaluated these transitions with Eq.~\eqref{eq:positive}~\eqref{eq:lie_derivative}. The results are reported in Table~\ref{tab:nlf_perf}. For any transition, $(s_t, s_{t+1})$, Eq.~\eqref{eq:positive} requires that $V(s_t)$ and $V(s_{t+1})$ should be greater than 0. Eq.~\eqref{eq:lie_derivative} requires that the lie derivative should be smaller than 0. A desired NLF should make all the transitions sampled from $\pi$ satisfy Eq.~\eqref{eq:positive}~\eqref{eq:lie_derivative}. We compared the NLF trained by our co-learning algorithm with the NLF trained in the post-learning phase~\cite{chang2020neural, xiong2022hisarl}. Because the co-learning algorithm can adapt the controller to the NLF during training, the NLFs trained by the co-learning algorithm have better quality with respect to satisfaction rate.

\subsection{Baselines}

Our approach explicitly avoids a robot entering a hazardous zone through the use of a runtime monitor. We compare two other existing model-free approaches that only implicitly encode avoidance behavior. We further evaluate these two algorithms with the RRT guidance for a fair comparison.

\paragraph{End-to-End (E2E)}
These RL controllers are trained with TD3 \cite{td3}.  Here safety violations are encoded into the reward via a penalty term. The reward at time $t$ is given by $r^{e2e}_t = r_t - C \mathds{1}_{haz}$, where $r_t$ is defined in Eq.~\eqref{eq:controller_reward}, and $\mathds{1}_{haz}$ is an indicator that shows whether a robot is in the hazardous zone. $C \in \mathbb{R}^+$ is a hyperparameter requiring tuning. The E2E is not guided by a planner and cannot avoid hazardous zones by following the planned path. Hence, we need to augment its observations with a vector $v_{haz} \in \mathbb{R}^{16}$ that provides position information of hazardous zones, allowing the controller to learn to avoid them. $v_{haz}$ contains $8$ vectors that point toward the $8$ closest hazardous zones. The input of the E2E controller is $[v_{haz}, s_t]$.

\paragraph{Constrained Policy Optimization}
Constrained Policy Optimization (CPO) \cite{cpo} is a trust-region method for solving the constrained MDP problem. By constraining safety violations, \cite{cpo} trains controllers in Safety Gym environments. However, the number of hazardous zones and map size described in the original CPO paper is limited compared with our approach. We observe that CPO can result in increasingly conservative behaviors as the number of hazardous zones and map sizes increase. CPO also needs to learn avoidance behaviors from information on hazardous zones directly. Thus, the input of the CPO controller is also $[v_{haz}, s_t]$.

\paragraph{Integrating Trained Controller with Planning}
Both the E2E and CPO are trained in the goal-conditioned state space, and as a consequence, we can also guide them with a sequence of waypoints to reach the final goal. We use the same RRT algorithm to generate the waypoints to guide the controllers trained with E2E and CPO. All the algorithms combined with the RRT planner are named H-XYZ, where XYZ is the algorithm used to train the controller.

\subsection{Safety, Reach Rate, and Performance}

\begin{table*}[ht]
  \centering
  \vspace{2px}
  \caption{Safety Violation and Reach Rate}
  E2E : End-to-End, CPO : Constrained Policy Optimization \cite{cpo}, H: Hierarchical Planner\\
  Safety violation and reach rate are fraction of 100 evaluation episodes with collisions and goal reachability, respectively. \\
  \label{tab:safety_and_reach_rate}
  \scriptsize
  \begin{tabular}{lc|ccccc|cccccc}\toprule
    \multirow{2}{*}{Difficulty} & \multirow{2}{*}{Robot} & \multicolumn{5}{c|}{Safety Violation $\downarrow$} & \multicolumn{5}{c}{Reach Rate $\uparrow$}                                                                               \\\cmidrule{3-12}
                                &                        & E2E                                  & CPO                            & H-E2E & H-CPO & H-Lyapunov    & E2E  & CPO  & H-E2E & H-CPO & H-Lyapunov    \\\midrule
    \multirow{4}{*}{level-1}    & \textit{Sweeping }     & 0.13                                 & 0.05                           & 0.27  & 0.32  & \textbf{0.01} & 0.87 & 0.75 & 0.73  & 0.68  & \textbf{0.99} \\
                                & \textit{Point    }     & 0.66                                 & 0.05                           & 0.08  & 0.21  & \textbf{0.04} & 0.25 & 0.07 & 0.72  & 0     & \textbf{0.96} \\
                                & \textit{Car      }     & 0.27                                 & 0.19                           & 0.41  & 0.18  & \textbf{0.00} & 0.73 & 0    & 0.59  & 0.22  & \textbf{1}    \\
                                & \textit{Quadruped}     & 0.44                                 & 0.09                           & 0.18  & 0.36  & \textbf{0.05} & 0.44 & 0    & 0.77  & 0.04  & \textbf{0.95} \\ \midrule
    \multirow{4}{*}{level-2}    & \textit{Sweeping }     & 0.51                                 & 0.12                           & 0.67  & 0.51  & \textbf{0.01} & 0.49 & 0.48 & 0.33  & 0.49  & \textbf{0.99} \\
                                & \textit{Point    }     & 0.42                                 & 0.1                            & 0.33  & 0.25  & \textbf{0.05} & 0.02 & 0    & 0.41  & 0     & \textbf{0.95} \\
                                & \textit{Car      }     & 0.61                                 & 0.21                           & 0.82  & 0.36  & \textbf{0.00} & 0.26 & 0    & 0.18  & 0.12  & \textbf{1}    \\
                                & \textit{Quadruped}     & 0.04                                 & 0.06                           & 0.35  & 0.53  & \textbf{0.04} & 0    & 0    & 0.41  & 0     & \textbf{0.96} \\ \midrule
    \multirow{4}{*}{level-3}    & \textit{Sweeping }     & 0.88                                 & 0.23                           & 0.83  & 0.84  & \textbf{0.03} & 0.12 & 0.17 & 0.15  & 0.14  & \textbf{0.97} \\
                                & \textit{Point    }     & 0.09                                 & 0.12                           & 0.45  & 0.33  & \textbf{0.06} & 0    & 0    & 0.09  & 0     & \textbf{0.94} \\
                                & \textit{Car      }     & 0.79                                 & 0.28                           & 0.98  & 0.38  & \textbf{0.02} & 0.02 & 0    & 0.02  & 0     & \textbf{0.98} \\
                                & \textit{Quadruped}     & 0.09                                 & 0.09                           & 0.44  & 0.59  & \textbf{0.08} & 0    & 0    & 0.1   & 0     & \textbf{0.92} \\
    \bottomrule
  \end{tabular}
  \vspace{-13px}
\end{table*}

We evaluate the trained controllers for the four robots on all difficulty levels. The max simulation steps for level-1 is 1,000; level-2 and level-3 are limited to 4,000 and 16,000 steps before termination, resp. When a robot enters any hazardous zone, we terminate the simulation immediately and report a safety violation.

By comparing the fraction of episodes with safety violations on deployment (Table \ref{tab:safety_and_reach_rate}), we observe a clear safety and reach rate benefit of our approach. In most cases, the safety violation ratios of all these algorithms are significantly higher than our approach. However, there also exist cases with low safety violations on other algorithms. Nevertheless, one should note that a low violation rate is independent of the policy's performance and goal reach rate. A low safety violation rate may also mean a non-performant or inactive controller. For example, the \textit{Quadruped} E2E controller under level-2 has a $0.04$ safety violation ratio, while it also cannot reach the goal in all the 100 episodes simulations. The \textit{point} H-E2E controller has a $0.08$ safety violation ratio with $0.72$ reach rate under level-1. However, as difficulty increases, its safety violation ratio dramatically increases, and its reach rate decreases as well. Our method is safer and has a significantly greater goal-reach rate than any of the baselines.

Although combining RRT with the controller can significantly benefit certain scenarios (e.g., for the E2E \textit{point} controller under level-1, the reach rate boosts from 0.25 to 0.72), Table~\ref{tab:safety_and_reach_rate} also shows that this is not a generally exploitable principle. Guided by the waypoints generated by RRT, the robot may have to visit more places before reaching the goal. These additional explorations can expose the robot to more hazardous zones. For example, for all the CPO \textit{Quadruped} controllers, the safety violation ratio increases significantly after combining them with RRT.

\begin{table}[ht]\centering
  \caption{Average Number of Steps to Reach Goal}\label{tab:performance}
  '-' indicates complete failure in task achievement.
  \scriptsize
  \begin{tabular}{lc|cccccc}\toprule
    \multirow{2}{*}{Difficulty} & \multirow{2}{*}{Robot} & \multicolumn{5}{c}{Steps to Reach}                                        \\\cmidrule{3-7}
                   level             &                        & E2E                                & CPO   & H-E2E  & H-CPO  & H-Lyapunov \\\midrule
    \multirow{4}{*}{1}    & \textit{Sweeping }     & 120.7                              & 121.4 & 201.7  & 192.1  & 234.7      \\
                                & \textit{Point    }     & 403.6                              & 793.8 & 409.6  & -      & 521.9      \\
                                & \textit{Car      }     & 381.1                              & -     & 444.2  & 765.5  & 417.0      \\
                                & \textit{Quadruped}     & 99.4                               & -     & 117.7  & 478.2  & 252.0      \\ \midrule
    \multirow{4}{*}{2}    & \textit{Sweeping }     & 325.8                              & 298.7 & 505.9  & 503.7  & 584.4      \\
                                & \textit{Point    }     & 528.0                              & -     & 887.1  & -      & 1015.2     \\
                                & \textit{Car      }     & 720.7                              & -     & 858.7  & 2297.6 & 812.1      \\
                                & \textit{Quadruped}     & -                                  & -     & 242.1  & -      & 500.5      \\ \midrule
    \multirow{4}{*}{3}    & \textit{Sweeping }     & 764.0                              & 675.4 & 1080.9 & 1097.9 & 1226.9     \\
                                & \textit{Point    }     & -                                  & -     & 1498.4 & -      & 1948.7     \\
                                & \textit{Car      }     & 1212.5                             & -     & 1678.0 & -      & 1573.8     \\
                                & \textit{Quadruped}     & -                                  & -     & 504.9  & -      & 880.5      \\
    \bottomrule
  \end{tabular}
\end{table}

We evaluate the performance of algorithms with the steps-to-goal in Table~\ref{tab:performance}. While our method yields an enhanced safety and goal reach rate, we pay the price of having a larger mean of the number of steps to reach in most environments. This is expected as a more reckless controller may reach the goal faster, but this would be at the cost of safety. Our algorithm has larger steps to reach because of the small RoA over-approximations when a robot passes narrow tunnels. These small over-approximations cause cautious behaviors and make the robot move slower than usual. We also find that after combining with RRT, the robot needs more steps to reach the goal when compared with its non-hierarchical counterpart. This is normal because the paths generated with RRT are not guaranteed to be the shortest. This problem can be alleviated with algorithms like $RRT^*$~\cite{lavalle2006planning}. However, since this paper only focuses on safety when a robot follows a planned path, the choice of planning algorithm is not our primary focus, and we leave this question for future work.

\section{Related Work}

Planning under kinodynamic constraints~\cite{donald1993kinodynamic} addresses the inconsistency problem between planning and control. However, most previous works~(\cite{tedrake2010lqr,li2016asymptotically,sakcak2019sampling,xie2015toward}) suffer from issues in scalability and generalization. Notably, compared with our work, \cite{tedrake2010lqr} has a similar idea that combines stability region and planning. However, this approach only works for linearized known dynamics, and computing the LQR-tree is computationally expensive for deployment. More recently, \cite{li2021mpc} introduces a learning-based path planner and follows the planned path with MPC solved by Cross-Entropy Method (CEM)~\cite{botev2013cross}. However, explicitly providing safety assurance is still challenging because the CEM avoids potential collisions via optimizing penalty terms encoded implicitly.

Integrating path planning with a DRL controller has been extensively explored recently (\cite{faust2018prm,chiang2019rl,ota2020efficient,chiang2019learning}). For instance, ~\cite{faust2018prm} and \cite{chiang2019rl} train a local point-to-point controller with DRL and employ probabilistic roadmaps and rapidly-exploring random trees as the high-level planner, respectively. \cite{ota2020efficient} applies a convolutional neural network planner and trains the reach-avoidance controller with DRL. \cite{chiang2019learning} learns end-to-end point-to-point and path-following navigation behaviors with DRL while it also enhances the DRL with an automatic parameter tuner. These approaches can work with unknown dynamics and high-dimensional raw sensor observations. However, the desired behaviors of these controllers are only specified in the reward functions and also inherit the limitations of standard RL. 



Safe reinforcement learning approaches~(\cite{cpo,chow2017risk,chow2018lyapunov,berkenkamp2017safe,cheng2019end,zhu2019inductive}) aim at learning controllers that causes limited constraint violations.  \cite{cpo, chow2017risk, chow2018lyapunov} can reduce violation numbers but also results in more conservative behaviors, which can affect performance. \cite{berkenkamp2017safe,cheng2019end,zhu2019inductive} can avoid all violations.  However, they require a safe controller and the dynamics to be known, which is a strong requirement compared to our model-free setting. Moreover, since these approaches are not integrated with any planner, they usually do not perform well for long-time-horizon navigation tasks.

The Lyapunov method is also applied in \cite{chow2018lyapunov,berkenkamp2017safe,cheng2019end,zhu2019inductive}. However, \cite{chow2018lyapunov} does not build an explicit Lyapunov function.  \cite{cheng2019end,zhu2019inductive} compute the Lyapunov function with analytic approaches instead of learning it. Hence, they suffer from scalability issues and require knowing the controller's dynamics. In contrast, modern approaches~\cite{chang2020neural,sun2020learning,petridis2006construction,richards2018lyapunov} models the Lyapunov function with a neural network, training it with the transitions sampled from the environment and controller. \cite{chang2020neural, petridis2006construction,richards2018lyapunov} build the NLF with a fixed controller, and thus unlike our work, cannot be used to improve the quality of the controller. \cite{sun2020learning} also co-learns the Lyapunov function and controller, but with supervised learning, requiring a training dataset in advance. 

\section{Conclusion}

In this paper, we firstly introduced the TNLF and subsequently, using information derived by modeling this function for a DRL controller, demonstrate a clear enhancement to the aspects of safety when following plans. Additionally, we prove that learning these controllers along with the TNLF can be done via a novel co-learning procedure yielding benefits to both components. We show that general planning algorithms albeit capable, are in themselves inherently lacking in their ability to avoid safety violations, and the execution strategies for these plans must also take into account the capabilities of the underlying controller. Our use of RoAs and runtime monitors is a significant leap towards realizing this synergy between safety and planning.

\bibliographystyle{IEEEtran}
\typeout{}
\bibliography{refs}

\end{document}